%% file: main.tex
\documentclass[10pt,letterpaper]{article}
\usepackage{spconf,amsmath,amssymb,amsthm,graphicx,booktabs}
\usepackage[T1]{fontenc}
\usepackage{newtxmath}
\usepackage[hidelinks]{hyperref}
\hypersetup{pdftitle={Reserve-Aware Contrast Certificates for Conservative Bandits with Uncertain Baselines},pdfsubject={Conservative contextual bandits},pdfkeywords={conservative bandits, confidence sets, safe exploration},pdflang={en-US},pdfauthor={Qinchuan Cheng}}
\usepackage{microtype}
\usepackage{algorithm,algpseudocode}
\newtheorem{proposition}{Proposition}
\newtheorem{theorem}{Theorem}
\newcommand{\R}{\mathbb{R}}
\newcommand{\calA}{\mathcal{A}}
\newcommand{\calC}{\mathcal{C}}
\newcommand{\norm}[1]{\left\lVert#1\right\rVert}
\newcommand{\joint}{\mathrm{J}}
\newcommand{\sep}{\mathrm{S}}
\input{tables/results_macros}
\title{Reserve-Aware Contrast Certificates for Conservative Bandits with Uncertain Baselines}
\input{authors}
\begin{document}
\maketitle
\begin{abstract}
Conservative bandits must improve an incumbent policy without exhausting a prescribed performance budget. When the incumbent is uncertain, separately bounding candidate and baseline rewards can charge twice for shared estimation error. We develop Reserve-C4B around the baseline-relative contrast itself. A shared confidence set yields an exact expression for this avoidable penalty and a tighter admissibility test at every fixed history. A reserve ledger separates statistical evidence from permitted performance deficit; a prefix-refresh extension recertifies accumulated decisions under the current confidence set without discarding previously certified credit. For linear rewards, self-normalized confidence sets provide simultaneous validity over time and adaptively generated candidates, and the resulting policy satisfies a conditional-mean performance constraint with high probability. Reproducible experiments isolate certificate coupling, prefix refresh, and historical information, showing large reductions in baseline fallback while exposing the limitations of frozen certificates.
\end{abstract}
\begin{keywords}
Conservative bandits, sequential decision making, confidence sets, baseline uncertainty, safe exploration
\end{keywords}

\section{Introduction}
Online adaptation in recommendation, sensing, and resource allocation must often retain the reliability of an existing policy. Conservative bandits formalize this requirement by constraining cumulative reward relative to an incumbent while learning from sequential feedback~\cite{wu2016conservative,kazerouni2017conservative}. With an uncertain incumbent, the safety decision is comparative: can an action's reward cover a specified fraction of the baseline's reward? Estimating these two quantities independently can obscure how much of their uncertainty is shared.

Consider two identical actions with a common reward interval $[\ell,U]$. A separate lower--upper comparison gives $\ell-(1-\alpha)U$, which may be negative even when $\ell>0$. The actual contrast is $\alpha\mu$, certified by $\alpha\ell$. The difference, $(1-\alpha)(U-\ell)$, is an uncertainty penalty created by the comparison rather than by the action. Nonidentical actions inherit the same problem when their estimation errors are aligned. A larger initial reserve can postpone rejection, but does not remove this penalty or validate an incorrect certificate.

We study this distinction at the certificate and budget levels. Reserve-C4B uses one confidence set to bound the candidate--baseline contrast, quantifies the price of decoupling that set, and spends only certified credit. Its prefix-refresh extension also pools uncertainty across executed rounds: historical decisions can be recertified as information improves, while a carry-forward bound preserves previously established credit. The analysis makes the source of every budget increment explicit.

\textbf{Relation to prior work.} Unknown baselines are already treated by conservative linear and combinatorial bandits~\cite{kazerouni2017conservative,zhang2019contextual}; improved exploration rules, budget reductions, and nonlinear regression oracles broaden that literature~\cite{garcelon2020improved,yang2021reduction,deb2024beyond}. Robust baseline regret also motivates evaluating two policies under a shared model~\cite{petrik2016safe}. Building on these formulations, we provide a contrast-first certificate analysis, an explicit carry/refresh ledger with an anytime safety proof, and matched ablations that identify where shared uncertainty and historical information change admissibility. Time-uniform off-policy evaluation offers a complementary route to gated deployment~\cite{karampatziakis2021offpolicy}; here the target is the realized action sequence's cumulative conditional means.

\section{Decision Setting and Statistical Evidence}
At round $t$, the learner observes a finite candidate set $\calA_t$, baseline $b_t$, and feature vectors $x_t(a)\in\R^d$. These are measurable before the current reward, conditional on the history and any candidate-generation randomness. Rewards follow
\begin{equation}
 Y_t=\mu_t(a_t)+\eta_t,\qquad \mu_t(a)=x_t(a)^\top\theta_* ,
 \label{eq:model}
\end{equation}
where $\theta_*$ is fixed, $\norm{\theta_*}_2\le S$, features are bounded and predictable, and $\eta_t$ is conditionally $\sigma$-sub-Gaussian with zero mean. Only the selected reward is observed. The baseline mean is unknown; its nonnegativity is known.

Let $c=1-\alpha$, $0<\alpha<1$, and
\begin{equation}
 z_t(a)=x_t(a)-cx_t(b_t),\quad \Delta_t(a)=z_t(a)^\top\theta_* .
\end{equation}
For an externally specified reserve $R_0\ge0$, the requirement is
\begin{equation}
 D_t:=R_0+\sum_{s=1}^t\Delta_s(a_s)\ge0
 \quad\text{for every }t. \label{eq:target}
\end{equation}
$R_0=0$ recovers the strict cumulative conservative constraint. Positive $R_0$ permits an additive deficit; it is not additional statistical confidence. Historical training observations improve estimation but receive no deployment budget credit. Equation~\eqref{eq:target} concerns conditional means, not every noisy realized-reward path.

Let $\mathcal I_t$ index all observations available before round $t$, including the historical sample. Form
\begin{align}
 V_t&=\lambda I+\sum_{i\in\mathcal I_t} x_ix_i^\top,\\
 \widehat\theta_t&=V_t^{-1}\sum_{i\in\mathcal I_t}x_iY_i,\\
 \beta_t&=\sigma\sqrt{\log\!\frac{\det V_t}{\lambda^d}+2\log\!\frac1\delta}
          +\sqrt\lambda S. \label{eq:beta}
\end{align}
The self-normalized result of~\cite{abbasi2011improved} gives
\begin{equation}
 \Pr\!\left\{\theta_*\in\calC_t\ \forall t\right\}\ge1-\delta,
 \quad\calC_t=\{\theta:\norm{\theta-\widehat\theta_t}_{V_t}\le\beta_t\}.
 \label{eq:event}
\end{equation}
This event covers every predictable candidate, even if the generator depends on past observations. No union bound over pool size is required. Marginal prediction coverage for future outcomes~\cite{lei2018distribution} is not a substitute for this simultaneous conditional-mean guarantee.

\section{Contrast Certificates and Reserve Ledgers}
\subsection{The cost of independent uncertainty bounds}
A shared-set certificate is $L_t^\joint(a)=\inf_{\theta\in\calC_t}z_t(a)^\top\theta$. Writing $\norm{v}_t=\sqrt{v^\top V_t^{-1}v}$, its closed form and the separate alternative are
\begin{align}
 L_t^\joint(a)&=z_t(a)^\top\widehat\theta_t-\beta_t\norm{z_t(a)}_t,\\
 L_t^\sep(a)&=z_t(a)^\top\widehat\theta_t
   -\beta_t\big(\norm{x_t(a)}_t+c\norm{x_t(b_t)}_t\big).
 \label{eq:certificates}
\end{align}
\begin{proposition}[Decoupling penalty]\label{prop:penalty}
For the same history and action,
\begin{equation}
 L_t^\joint(a)-L_t^\sep(a)=\beta_t\big(\norm{x_t(a)}_t
 +c\norm{x_t(b_t)}_t-\norm{z_t(a)}_t\big)\ge0.\label{eq:penalty}
\end{equation}
At $a=b_t$, the penalty equals $c(U_t(b_t)-\ell_t(b_t))$, where $\ell_t,U_t$ are the ellipsoid's reward bounds.
\end{proposition}
\noindent\emph{Proof.} Minimizing a linear functional over the ellipsoid gives~\eqref{eq:certificates}. Subtraction gives~\eqref{eq:penalty}, and the triangle inequality establishes nonnegativity. When $a=b_t$, $z_t=\alpha x_t(b_t)$ and $U_t-\ell_t=2\beta_t\norm{x_t(b_t)}_t$.\hfill$\square$

Thus a shared certificate admits every action admitted by a separate certificate \emph{at a fixed history and bank}. This does not imply reward dominance between policies that subsequently collect different data. For the baseline itself, both methods should exploit its identity and known nonnegativity:
\begin{equation}
 L_t(b_t)=\alpha\max\{\ell_t(b_t),0\}\ge0.\label{eq:fallback}
\end{equation}
Rejecting the baseline using the separate formula would be an avoidable implementation error, not an intrinsic obstacle to conservative learning.

\subsection{Carry-forward and prefix refresh}
Initialize $B_0=R_0$. The frozen Reserve-C4B ledger tests $G_t^{\rm carry}(a)=B_{t-1}+L_t^\joint(a)$, using~\eqref{eq:fallback} for $b_t$. Accept the highest-UCB nonbaseline candidate with $G_t^{\rm carry}(a)\ge0$, or execute the baseline if none qualifies. The score is $x_t(a)^\top\widehat\theta_t+\beta_t\norm{x_t(a)}_t$. Set $B_t=G_t^{\rm carry}(a_t)$ before observing $Y_t$.

Frozen certificates cannot recover credit lost to early uncertainty. Under the shared parameter in~\eqref{eq:model}, keep $Z_{t-1}=\sum_{s<t}z_s(a_s)$ and recertify the complete candidate prefix:
\begin{equation}
 Q_t(a)=R_0+(Z_{t-1}+z_t(a))^\top\widehat\theta_t
             -\beta_t\norm{Z_{t-1}+z_t(a)}_t.\label{eq:refresh}
\end{equation}
The refresh extension uses
\begin{equation}
 G_t(a)=\max\{B_{t-1}+L_t(a),Q_t(a)\},\qquad B_t=G_t(a_t),\label{eq:gate}
\end{equation}
with the same UCB selection rule, then updates $Z_t$, observes $Y_t$, and updates the estimator. Refreshing uses newly available evidence, not unearned reward credit. The maximum is essential: a current confidence set need not be nested inside yesterday's set, so a recomputed bound can be worse than the carried certificate.

\begin{theorem}[Anytime conditional-mean safety]\label{thm:safety}
Under~\eqref{eq:model}--\eqref{eq:event}, both ledgers satisfy $D_t\ge B_t\ge0$ at every deployment round with probability at least $1-\delta$.
\end{theorem}
\noindent\emph{Proof.} On the common event~\eqref{eq:event}, $L_t(a_t)\le\Delta_t(a_t)$. If $B_{t-1}\le D_{t-1}$, the carry bound is at most $D_t$. Also, $Q_t(a_t)$ lower-bounds $R_0+Z_t^\top\theta_*=D_t$. Their maximum remains a lower bound. The gate ensures nonnegativity, and the baseline is always feasible by~\eqref{eq:fallback}. Induction starts at $B_0=D_0=R_0$.\hfill$\square$

\begin{algorithm}[t]
\caption{Reserve-C4B with optional prefix refresh}
\label{alg:policy}
\begin{algorithmic}[1]
\State Initialize estimator from history; $B\gets R_0$, $Z\gets0$.
\For{deployment rounds $t=1,\ldots,T$}
  \State Observe candidates and baseline; form $\calC_t$.
  \State Compute UCB scores and contrast certificates.
  \State Set $G(a)\gets B+L_t(a)$ for every action.
  \If{refresh is enabled}
    \State $G(a)\gets\max\{G(a),Q_t(a)\}$ using $Z$.
  \EndIf
  \State Choose highest-UCB candidate with $G(a)\ge0$;
  \Statex \hspace{\algorithmicindent}\hspace{\algorithmicindent}use baseline when no candidate qualifies.
  \State Log $B\gets G(a_t)$; set $Z\gets Z+z_t(a_t)$.
  \State Execute $a_t$, observe $Y_t$, update estimator.
\EndFor
\end{algorithmic}
\end{algorithm}

The proof remains valid for adaptively chosen candidates because a single parameter-containment event covers all of them. Without a certified nonnegative fallback, the policy must stop when no action qualifies. With dense covariance matrices, candidate evaluation costs $O(|\calA_t|d^2)$ per round; rank-one updates cost $O(d^2)$. Refresh adds only the $d$-vector $Z_t$, not storage of the full trajectory.

\subsection{How much reserve does decoupling consume?}
For a fixed action path and its frozen certificates $L_1,\ldots,L_T$, define the smallest reserve that makes every certified prefix nonnegative. It has the exact form
\begin{equation}
 R_{\min}(L)=\max\left\{0,\max_{t\le T}\left[-\sum_{s\le t}L_s\right]\right\}.
 \label{eq:minreserve}
\end{equation}
\begin{proposition}[Fixed-path reserve cost]\label{prop:reserve}
Let $\Pi_s=L_s^\joint-L_s^\sep\ge0$ along the same path and confidence sequence, using the same structural baseline bound in both ledgers. Then
\begin{equation}
 0\le R_{\min}(L^\sep)-R_{\min}(L^\joint)\le\sum_{s=1}^T\Pi_s.
 \label{eq:reservecost}
\end{equation}
\end{proposition}
\noindent\emph{Proof.} Every separate-bound deficit equals the corresponding joint-bound deficit plus $\sum_{s\le t}\Pi_s$. Taking prefix maxima proves the lower inequality; bounding every prefix penalty by its total proves the upper inequality. Equation~\eqref{eq:minreserve} follows directly from the prefix constraints.\hfill$\square$

This is a certificate-level reserve requirement, not the reserve an online policy can know in advance. It explains why decoupling can consume an entire exploration allowance even when each reward interval is individually valid. Prefix refresh addresses the remaining cost of freezing early certificates; it changes the ledger, rather than simply increasing $R_0$.

\section{Experiments}
\subsection{Reproducible protocol and comparisons}
We use a controlled linear study to separate uncertainty geometry from model misspecification. Set $d=5$, $\theta_*=(1,0.6,0,0,0)$, $x(b)=(1,0,0,0,0)$, and draw 32 candidates $x(a)=(1,\rho u)$ per round, with $u$ uniform on the unit sphere in $\R^4$. Gaussian noise has standard deviation $\sigma$. The learner knows $S=1.5$ and $\sigma$, but not $\theta_*$ or the baseline mean. With $\rho\le0.8$, all means are nonnegative.

Each episode has 20 historical observations followed by 200 deployment rounds. Histories are either diverse, with features $(1,u)$, or baseline-only. We test every combination of $\rho\in\{0.15,0.4,0.8\}$, $\sigma\in\{0.1,0.3\}$, and $R_0\in\{0,0.5,2\}$ under both histories: 36 settings, 256 independent episodes per setting, with $\alpha=\delta=0.05$ and $\lambda=0.1$. Methods share historical samples, candidate pools, and potential noise within each episode. Scenario reuse is not treated as independent replication. Intervals use Student's $t$ over episodes, pairing method differences within episodes.

\textbf{Separate} and \textbf{Contrast} use the frozen ledger with $L^\sep$ and $L^\joint$, respectively. \textbf{Refresh} adds~\eqref{eq:gate}. \textbf{Revalue} applies the unknown-baseline cumulative separate-bound geometry of~\cite{kazerouni2017conservative}, checking only the optimistic proposal; \textbf{Revalue-F} instead selects the best admissible proposal. Both use the current ellipsoid and structural fallback, and are explicit ablations rather than reproductions of the published nested-set algorithm. \textbf{LinUCB} has no safety gate. Revalue-F controls for action filtering when comparing refresh strategies.

\begin{table}[t]
\centering
\caption{Mean reward ratio / fallback percentage at $\rho=0.15$, $\sigma=0.3$, $R_0=0$ (256 episodes). Reward is normalized by the true baseline mean, used only by the evaluator.}
\label{tab:central}
\input{tables/central}
\end{table}

\begin{figure}[t]
\centering
\includegraphics[width=\columnwidth]{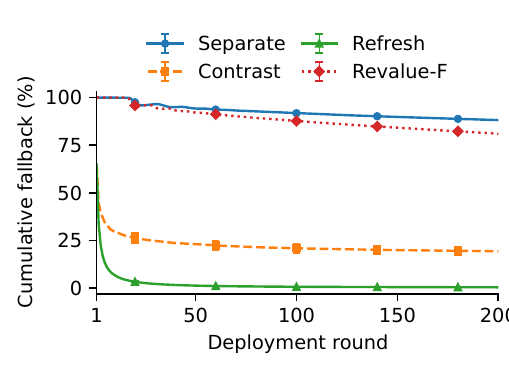}
\caption{Cumulative fallback under diverse history ($\rho=0.15$, $\sigma=0.3$, $R_0=0$). Curves average 256 episodes; error bars mark pointwise 95\% intervals at selected rounds.}
\label{fig:trajectory}
\end{figure}

\subsection{What coupling and refresh change}
Table~\ref{tab:central} isolates the two effects. With diverse history, replacing Separate by Contrast reduces fallback from \SepFallback\% to \ContrastFallback\% and increases the reward ratio by \ContrastGain{} (paired 95\% interval: \ContrastGainLow{}--\ContrastGainHigh). Refresh reduces fallback further to \RefreshFallback\%. Figure~\ref{fig:trajectory} shows that this is sustained throughout deployment rather than confined to initialization.

Baseline-only history exposes the limitation of a frozen ledger. Contrast still falls back on \BaselineContrastFallback\% of rounds: repeatedly observing the incumbent does not identify all action directions, and previous pessimistic charges remain in the ledger. Refresh reduces fallback to \BaselineRefreshFallback\%, with reward ratio \BaselineRefreshReward{} versus \BaselineRevalueFilterReward{} for Revalue-F. Its paired reward gain over Revalue-F is \RefreshVsRevalueGain{} (95\% interval: \RefreshVsRevalueLow{}--\RefreshVsRevalueHigh). This comparison separates shared-prefix certification from merely choosing a different admissible action.

\begin{figure}[t]
\centering
\includegraphics[width=\columnwidth]{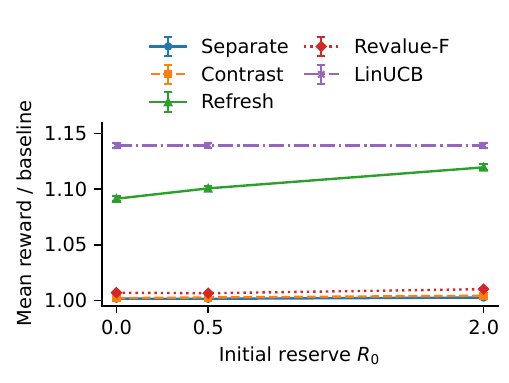}
\caption{Reserve sensitivity with baseline-only history ($\rho=0.4$, $\sigma=0.3$). Points are means with pointwise 95\% intervals over 256 episodes; lines connect the tested reserve values.}
\label{fig:reserve}
\end{figure}

\begin{table}[t]
\centering
\caption{Zero-reserve reward ratios across noise levels and candidate radii (256 episodes per setting). C: Contrast; C+: Refresh; R-F: Revalue-F; U: LinUCB.}
\label{tab:grid}
\input{tables/zero_reserve}
\end{table}

Table~\ref{tab:grid} varies candidate radius and observation noise without changing the reserve. Refresh improves on the frozen contrast ledger in every displayed setting. Its gap to unconstrained LinUCB is larger with baseline-only history and wider candidate variation, reflecting the cost of learning unfamiliar directions under a prefix constraint. The revalued comparison is also informative: increasing the admissible choice set can change the data collected, so Revalue-F need not outperform the optimistic-proposal Revalue policy on reward. Fixed-history inclusion is a geometric statement, not a ranking of complete learning trajectories.

Increasing reserve relaxes the performance requirement and can improve exploration (Fig.~\ref{fig:reserve}), but does not eliminate frozen-certificate conservatism. The baseline-only setting shows why reserve and historical information must be varied separately.

\textbf{Safety and scope.} All five gated variants had zero conditional-mean prefix violations in all tested settings, and all evaluated confidence events held. Unconstrained LinUCB violated the zero-reserve constraint in \LinBaselineViolations\% of episodes in Table~\ref{tab:central}'s baseline-only setting. These observations check implementation behavior; the guarantee comes from Theorem~\ref{thm:safety}, not from empirical coverage alone. Zero failures in 256 episodes imply a one-sided 95\% binomial upper bound of \FailureUpper\% for one fixed setting, not a simultaneous guarantee across the grid. The study does not establish robustness to nonlinear rewards, drifting parameters, misspecified noise bounds, or performance on live retail traffic. The comparisons isolate certificate mechanisms at a fixed horizon.

\section{Conclusion}
Conservative decisions should be certified in the same comparative coordinates as their performance constraint. Shared contrast certificates remove a precisely quantifiable uncertainty penalty; prefix refresh prevents early pessimism from becoming permanent budget debt. Together, they give a transparent reserve-aware filter with an anytime conditional-mean guarantee and measurable reductions in fallback under a fully reproducible linear model.

\par\medskip
{\centering\bfseries ACKNOWLEDGMENTS\par}
\nobreak\smallskip\noindent
ChatGPT (OpenAI) assisted with writing and idea refinement throughout the paper, and with analysis, code, and figures in Sections~2--4.

\clearpage
\bibliographystyle{IEEEbib}
\bibliography{references}
\end{document}

%% file: tables/results_macros.tex
\newcommand{\SepFallback}{88.2}
\newcommand{\ContrastFallback}{19.4}
\newcommand{\RefreshFallback}{0.5}
\newcommand{\ContrastGain}{0.0473}
\newcommand{\ContrastGainLow}{0.0456}
\newcommand{\ContrastGainHigh}{0.0490}
\newcommand{\BaselineContrastFallback}{91.8}
\newcommand{\BaselineRefreshFallback}{12.1}
\newcommand{\BaselineRefreshReward}{1.0190}
\newcommand{\BaselineRevalueFilterReward}{1.0011}
\newcommand{\RefreshVsRevalueGain}{0.0179}
\newcommand{\RefreshVsRevalueLow}{0.0169}
\newcommand{\RefreshVsRevalueHigh}{0.0190}
\newcommand{\LinBaselineViolations}{16.4}
\newcommand{\FailureUpper}{1.16}

%% file: authors.tex
\input{author_metadata}
\name{\AuthorName}
\address{\AuthorSchool, \AuthorUniversity, \AuthorLocation\\
\href{mailto:\AuthorEmail}{\AuthorEmail}\quad
ORCID: \href{https://orcid.org/\AuthorORCID}{\AuthorORCID}}

%% file: author_metadata.tex
\newcommand{\AuthorName}{Qinchuan Cheng}
\newcommand{\AuthorSchool}{School of Automation Science and Engineering}
\newcommand{\AuthorUniversity}{Xi'an Jiaotong University}
\newcommand{\AuthorLocation}{Xi'an, China}
\newcommand{\AuthorEmail}{2310820636@qq.com}
\newcommand{\AuthorORCID}{0009-0005-9325-7554}

%% file: tables/central.tex
\begin{tabular}{lrr}
\toprule
Method & Diverse history & Baseline-only \\
\midrule
Separate & 1.0084 / 88.2 & 1.0007 / 94.1 \\
Contrast & 1.0557 / 19.4 & 1.0011 / 91.8 \\
\textbf{Refresh} & 1.0606 / 0.5 & 1.0190 / 12.1 \\
Revalue & 1.0124 / 81.6 & 1.0025 / 83.8 \\
Revalue-F & 1.0121 / 81.1 & 1.0011 / 80.0 \\
LinUCB & 1.0613 / 0.0 & 1.0272 / 0.0 \\
\bottomrule
\end{tabular}

%% file: tables/zero_reserve.tex
\begin{tabular}{rrrrrr}
\toprule
$\rho$ & $\sigma$ & C & C+ & R-F & U \\
\midrule
\multicolumn{6}{l}{\emph{Diverse history}} \\
0.15 & 0.1 & 1.077 & 1.077 & 1.051 & 1.077 \\
0.15 & 0.3 & 1.056 & 1.061 & 1.012 & 1.061 \\
0.4 & 0.1 & 1.205 & 1.205 & 1.183 & 1.205 \\
0.4 & 0.3 & 1.099 & 1.164 & 1.038 & 1.173 \\
0.8 & 0.1 & 1.414 & 1.414 & 1.401 & 1.415 \\
0.8 & 0.3 & 1.218 & 1.358 & 1.175 & 1.377 \\
\midrule
\multicolumn{6}{l}{\emph{Baseline-only history}} \\
0.15 & 0.1 & 1.008 & 1.047 & 1.015 & 1.048 \\
0.15 & 0.3 & 1.001 & 1.019 & 1.001 & 1.027 \\
0.4 & 0.1 & 1.018 & 1.163 & 1.121 & 1.184 \\
0.4 & 0.3 & 1.002 & 1.091 & 1.007 & 1.139 \\
0.8 & 0.1 & 1.045 & 1.311 & 1.297 & 1.397 \\
0.8 & 0.3 & 1.003 & 1.198 & 1.065 & 1.350 \\
\bottomrule
\end{tabular}